\documentclass{article}
\usepackage{fullpage, amsmath, amsthm, amssymb}

\usepackage[noend,algoruled,linesnumbered,]{algorithm2e}
\newcommand{\Dist}{\mathcal{D}}
\newcommand{\Xs}{\mathcal{X}}
\DeclareMathOperator*{\sign}{sign}
\newcommand{\Loss}{\mathcal{L}}

\newcommand{\Hyp}{\mathbb{C}}

\newcommand{\Buc}{\mathcal{C}}
\newcommand{\rU}{\textbf{U}}
\newcommand{\rV}{\textbf{V}}
\newcommand{\rInto}{\textbf{R}}
\newcommand{\Fam}{\mathcal{F}}
\newcommand{\E}{\mathbb{E}}
\newcommand{\Tree}{\mathcal{T}}
\newcommand{\Ms}{\mathcal{M}}
\newcommand{\HanI}{\mathcal{I}}

\newcommand{\rInd}{\textbf{I}}
\newcommand{\eps}{\varepsilon}
\newcommand{\rS}{\textbf{S}}

\newcommand{\rx}{\textbf{x}}

\newcommand{\Alg}{\mathcal{A}}
\newcommand{\Distinct}{D}
\newcommand{\ri}{\textbf{i}}
\newcommand{\Sub}{\mathcal{W}}

\newcommand{\sSet}{\mathcal{E}}
\newcommand{\rBuc}{\textbf{C}}
\newcommand{\Pdist}{\mathcal{P}}

\newcommand{\Boots}{\textbf{B}}

\newtheorem{theorem}{Theorem}
\newtheorem{lemma}{Lemma}

\newtheorem{corollary}{Corollary}

\title{Bagging is an Optimal PAC Learner}

\author{Kasper Green Larsen\thanks{\texttt{larsen@cs.au.dk}, Supported by Independent Research Fund Denmark (DFF) Sapere Aude Research Leader grant No 9064-00068B.}\\Aarhus University}

\begin{document}

\date{}
\maketitle

\begin{abstract}
  Determining the optimal sample complexity of PAC learning in the realizable setting was a central open problem in learning theory for decades. Finally, the seminal work by Hanneke (2016) gave an algorithm with a provably optimal sample complexity. His algorithm is based on a careful and structured sub-sampling of the training data and then returning a majority vote among hypotheses trained on each of the sub-samples. While being a very exciting theoretical result, it has not had much impact in practice, in part due to inefficiency, since it constructs a polynomial number of sub-samples of the training data, each of linear size.

  In this work, we prove the surprising result that the practical and classic heuristic \emph{bagging} (a.k.a. bootstrap aggregation), due to Breiman (1996), is in fact also an optimal PAC learner. Bagging pre-dates Hanneke's algorithm by twenty years and is taught in most undergraduate machine learning courses. Moreover, we show that it only requires a logarithmic number of sub-samples to reach optimality.
\end{abstract}

\section{Introduction}
\label{sec:intro}
PAC learning, or probably approximately correct learning (Valiant~\cite{valiant1984theory}), is the most
classic theoretical model for studying classification problems in
supervised learning. For binary classification in the \emph{realizable} setting, the goal is to design a learning algorithm
that with probability $1-\delta$ over a random training data set, outputs a
hypothesis that mispredicts the label of a new random sample with probability
at most $\eps$. More formally, one assumes that samples come from an input
domain $\Xs$ and that there is an unknown \emph{concept} $c : \Xs \to
\{-1,1\}$ that we are trying to learn. The realizable setting means that
$c$ belongs to a predefined concept class $\Hyp \subseteq \Xs \to \{-1,1\}$ and that the
correct label of any $x \in \Xs$ is always $c(x)$. 

For the above learning task, a learning
algorithm $\Alg$ receives a training data set $\rS$ of $m$ i.i.d. samples
$(\rx_1,c(\rx_1)),\dots,(\rx_m,c(\rx_m))$ where each $\rx_i$ is drawn
independently from an \emph{unknown} data distribution
$\Dist$ over $\Xs$. From this data set, the learning algorithm
must output a hypothesis $h_\rS : \Xs \to \{-1,1\}$. The algorithm $\Alg$ is a
PAC learner, if for any distribution $\Dist$ and any concept $c \in \Hyp$, it holds that if
$\Alg$ is given enough i.i.d. training samples $\rS$, then with
probability at least $1-\delta$, the hypothesis $h_\rS$
that it outputs satisfies $\Loss_\Dist(h_\rS) = \Pr_{\rx \sim \Dist}[h(\rx) \neq c(\rx)] \leq
\eps$. We remark that the algorithm $\Alg$ knows the
concept class $\Hyp$, but not the
data distribution $\Dist$. Determining the minimum number of samples
$\Ms(\eps,\delta)$, as a function of $\eps$, $\delta$ and the
VC-dimension~\cite{vcdim} $d$ of $\Hyp$ (see Section~\ref{sec:overview} for
a formal definition of VC-dimension) needed for
this learning task, is one of the fundamental problems in PAC learning.

The most natural learning algorithm for the above is \emph{empirical
  risk minimization} (ERM). Here a learning algorithm simply outputs
an arbitrary hypothesis/concept $h_\rS \in \Hyp$ that correctly predicts the
labels of the training data, i.e. it has $h_\rS(\rx_i) = c(\rx_i)$ for all
$(\rx_i,c(\rx_i)) \in \rS$. Clearly such a hypothesis exists since $c \in
\Hyp$. Such a learning algorithm is referred to as a \emph{proper}
learner as it outputs a hypothesis/concept from the concept class $\Hyp$. ERM is known to obtain a
sample complexity of
$O(\eps^{-1}(d \lg(1/\eps) + \lg(1/\delta)))$~\cite{vapnik:estimation,blumer1989learnability}. Moreover, it
can be shown that this analysis cannot be tightened, i.e. there are
distributions $\Dist$ and concept classes $\Hyp$ where any proper
learner needs $\Omega(\eps^{-1}(d \lg(1/\eps) +
\lg(1/\delta)))$ samples~\cite{properLB}. However, a PAC learning algorithm is not necessarily
required to output a hypothesis $h \in \Hyp$. That better strategies
might exist may seem counter-intuitive at first, since we are
promised that the unknown concept $c$ lies in $\Hyp$. Nonetheless, the
strongest known lower bounds for arbitrary PAC learning algorithms only
show that $\Omega(\eps^{-1}(d + \lg(1/\delta)))$ samples are
necessary~\cite{blumer1989learnability,ehrenfeucht1989general}. This leaves a gap of a factor $\lg(1/\eps)$ between ERM and
the lower bound for arbitrary algorithms.

Despite its centrality, closing this gap
remained a big open problem for more than thirty years. Finally, in
2016, Hanneke~\cite{hanneke2016optimal} built on ideas by Simon~\cite{simon2015almost} and presented an algorithm with an asymptotically
optimal sample complexity of
$\Ms(\eps,\delta) = O(\eps^{-1}(d + \lg(1/\delta)))$. His
algorithm is based on constructing a number of subsets
$\rS_i \subset \rS$ of the training data $\rS$ with carefully designed
overlaps between the $\rS_i$'s (see Section~\ref{sec:overview}). He
then runs ERM on each $\rS_i$ to obtain hypotheses $h_{\rS_i} \in \Hyp$ and
finally outputs the hypothesis $f_\rS$ taking the majority vote $f_\rS(x) =
\sign(\sum_i h_{\rS_i}(x))$ among the $h_{\rS_i}$'s.

While being a major theoretical breakthrough, Hanneke's algorithm has
unfortunately not had any significant practical impact. One
explanation is that it requires a rather large number of sub-samples
$\rS_i$. Concretely, with optimal $m=\Theta(\eps^{-1}(d + \lg(1/\delta)))$ samples, it requires $m^{\lg_4 3} \approx m^{0.79}$
sub-samples of linear size $|\rS_i| = \Omega(m)$, resulting in a
somewhat slow learning algorithm.

\paragraph{Our Contribution.}
In this work, we present an alternative optimal PAC learner in the
realizable setting. Surprisingly, our algorithm is not new, but
actually pre-dates Hanneke's algorithm by twenty years. Concretely, we
show that the heuristic known as bagging (bootstrap aggregation) by
Breiman~\cite{bagging}, also gives an optimal PAC learner. Bagging, and its
slightly more involved extension known as \emph{random forest}~\cite{randomforest}, have
proved very efficient in practice and are classic topics in
introduction to machine learning courses.

In bagging, for $t$
iterations, we sample a subset $\rS_i$ of $n$ independent and uniform samples with replacement
from $\rS$. We then run ERM on each $\rS_i$ to produce hypotheses
$h_{\rS_i} \in \Hyp$ and finally output the hypothesis
$f_{\rS_1,\dots,\rS_t}$ taking the majority vote $f_{\rS_1,\dots,\rS_t }(x) = \sign(\sum_i h_{\rS_i}(x))$ among the
$h_{\rS_i}$'s. The sub-samples $\rS_i$ are referred to as \emph{bootstrap
  samples}.

While being similar to Hanneke's algorithm, it is 
simpler to construct the subsets $\rS_i$, and also, we show that it
suffices with just $t=O(\lg(m/\delta))$ bootstrap samples of size $n$ for any
$0.02m \leq n \leq m$, which should be compared to Hanneke's algorithm
requiring $m^{0.79}$ subsets (where $m$ is optimal $\Theta(\eps^{-1}(d
+ \lg(1/\delta)))$). For ease of notation, let $\Dist_c$ denote the
distribution of a pair $(\rx,c(\rx))$ with $\rx \sim \Dist$. Our result is then formalized in the
following theorem
\begin{theorem}
  \label{thm:mainintro}
There is a universal constant $a > 0$ such that for every $0 <
\delta < 1$, every distribution $\Dist$ over an input domain $\Xs$, every concept
class $\Hyp \subseteq \Xs \to \{-1,1\}$ of VC-dimension $d$ and every
$c \in \Hyp$, if $t \geq
18 \ln(2m/\delta)$ and $0.02m \leq n \leq m$, then it holds with
probability at least $1-\delta$ over the random choice of a training
set $\rS \sim \Dist_c^m$ and $t$ bootstrap samples $\rS_1,\dots,\rS_t \subset \rS$ of size $n$,
that the hypothesis $f_{\rS_1,\dots,\rS_t }$ produced by bagging satisfies
  \[
    \Loss_\Dist(f_{\rS_1,\dots,\rS_t }) = \Pr_{\rx \sim \Dist}[f_{\rS_1,\dots,\rS_t }(\rx) \neq c(\rx)] \leq a \cdot \frac{d + \ln(1/\delta)}{m}.
    \]
\end{theorem}
Solving for $m$ such that $\eps \leq   \Loss_\Dist(g)$ gives a sample complexity of
$O(\eps^{-1}(d + \lg(1/\delta))$ as claimed. We remark that the
constants $18$ and $0.02$ can be reduced at the cost of increasing the
unspecified constant $a$.

In addition to providing an alternative and simpler algorithm for
optimal PAC learning, we also believe there is much value in providing
further theoretical justification for the wide practical success of
bagging. 

As an interesting secondary result, if we combine our proof with a
recent work by Larsen and Ritzert~\cite{Larsen2022OptimalWT}, we get
that bagging combined with a version of AdaBoost gives an optimal weak
to strong
learner~\cite{kearns1988learning,kearns1994cryptographic}. We find it
quite remarkable that the combination of the classic techniques bagging
and boosting yields this result.

In Section~\ref{sec:overview}, we first present Hanneke's optimal PAC learner
and highlight the main ideas in his analysis. We then proceed to
present a high-level overview of our proof of Theorem~\ref{thm:mainintro}, which
re-uses some of the ideas from Hanneke's proof. In Section~\ref{sec:proof} we
proceed to give the formal proof of
Theorem~\ref{thm:mainintro}. Finally, in Section~\ref{sec:conclusion},
we comment further on the weak to strong
learning result, as well as discuss some directions for further research.

\section{Proof overview}
\label{sec:overview}
In this section, we first present Hanneke's PAC learning algorithm and discuss the main ideas in his analysis. We then proceed to give a high-level overview of the keys ideas in our proof that bagging is also an optimal PAC learning algorithm. For completeness, we start by recalling the definition of Vapnik-Chervonenkis dimension~\cite{vcdim}, or VC-dimension for short.

A concept class $\Hyp \subseteq \Xs \to \{-1,1\}$ has VC-dimension $d$, where $d$ is the largest integer such that there exists a set of $d$ samples $x_1,\dots,x_d \in \Xs$ for which any labeling of the $d$ samples can be realized by a concept $c \in \Hyp$. That is, $|\{(c(x_1),\dots,c(x_d)) : c \in \Hyp\}| = 2^d$. Throughout the paper, we assume $d \geq 1$ which is always true when $\Hyp$ contains at least two distinct concepts. Also, as the reader may have observed, we consistently use \textbf{bold} face letters to denote random variables.

\paragraph{Hanneke's Algorithm and Analysis.}
As mentioned in Section~\ref{sec:intro}, Hanneke's algorithm constructs a carefully selected collection of sub-samples of the training set $\rS$. These sub-samples are constructed by invoking Algorithm~\ref{alg:hanneke} as \textit{Sub-Sample}($\rS,\emptyset$).

\begin{algorithm}
  \DontPrintSemicolon
  \KwIn{Two sets of training samples $U,V$. 
  }
  \KwResult{Collection consisting of sub-samples.}
  \eIf{$|U| < 4$}{
    \Return{$\{ U \cup V\}$.}
    }
    {
      Partition $U$ into $4$ disjoint sets $U_0,U_1,U_2,U_3$ of $|U|/4$ samples each.\\
      \Return{$\bigcup_{i=1}^{3}$Sub-Sample$(U_0, V \cup (\bigcup_{j \in \{1,2,3\} \setminus \{i\}} U_j))$.} 
    }
  \caption{\textit{Sub-Sample}($U,V$)}\label{alg:hanneke}
\end{algorithm}

For simplicity, we have presented the \textit{Sub-Sample} algorithm assuming that $m$ is a power of $4$. 

Since $|U|$ is reduced by a factor $4$ in each recursive call and there are $3$ such calls, we get that the total number of sub-samples produced is $3^{\lg_4 m} = m^{\lg_4 3} \approx m^{0.79}$.

To analyse the hypothesis produced by invoking \textit{Sub-Sample}($\rS,\emptyset$), running ERM on each produced sub-sample $\rS_i$ to produce hypotheses $h_{\rS_i}$ and taking a majority vote $f_\rS(x) = \sign(\sum_i h_{\rS_i}(x))$, we zoom in on the recursive invocations of \textit{Sub-Sample}($\rU,\rV$). Such an invocation produces a number of sub-samples $\rS_1,\dots,\rS_{t}$ of $\rU \cup \rV \subseteq \rS$. Letting $h_{\rS_1},\dots,h_{\rS_{t}}$ denote the hypotheses obtained by running ERM on these sub-samples and $f_{(\rU,\rV)}(x) = \sign(\sum_i h_{\rS_i}(x))$ the majority vote among them, we prove by induction (with the base being the leaves of the recursion) that with probability at least $1-\delta$ over $\rU$, it holds that $\Loss_\Dist(f_{(\rU,\rV)}) \leq a \cdot (d + \ln(1/\delta))/|\rU|$ for a universal constant $a > 0$. If we can complete this inductive step, then the conclusion follows by examining the root invocation \textit{Sub-Sample}($\rS,\emptyset$).

The base case in the inductive proof is simply when $a \cdot (d + \ln(1/\delta))/|\rU| > 1$. Here the conclusion follows trivially as we always have $\Loss_\Dist(f_{(\rU,\rV)}) \leq 1$. For the inductive step, let $f_{1,(\rU,\rV)}, f_{2,(\rU,\rV)}$ and $f_{3,(\rU,\rV)}$ denote the majority voters produced by the three recursive calls \textit{Sub-Sample}($\rU_0,\rV \cup \rU_2 \cup \rU_3$), \textit{Sub-Sample}($\rU_0,\rV \cup \rU_1 \cup \rU_3$) and \textit{Sub-Sample}($\rU_0,\rV \cup \rU_1 \cup \rU_2$). Each of these have $\Loss_\Dist(f_{i,(\rU,\rV)}) \leq a \cdot (d + \ln(1/\delta))/|\rU_i| =  4a \cdot (d + \ln(1/\delta))/|\rU|$ by the induction hypothesis (except with some probability  $\delta$ which we ignore here for simplicity). For short, we say that a hypothesis $h$ errs on $x$ if $h(x) \neq c(x)$. The crux of the argument is now to show that it is very unlikely that $f_{i,(\rU,\rV)}$ errs on a sample $\rx \sim \Dist$ at the same time as a hypothesis $h_{\rS'}$ trained on a sub-sample $\rS'$ produced by a recursive call $j \neq i$ also errs on $\rx$.

For this step, let us wlog. consider the majority vote $f_{1,(\rU,\rV)}$ (over hypotheses produced by ERM on \textit{Sub-Sample}($\rU_0,\rV \cup \rU_2 \cup \rU_3$)). Intuitively, if $\Loss_\Dist(f_{1,(\rU,\rV)}) \ll a \cdot (d + \ln(1/\delta))/|\rU|$ then the hypotheses in $f_{1,(\rU,\rV)}$ contribute little to $\Loss_\Dist(f_{(\rU,\rV)})$. So assume instead $\Loss_\Dist(f_{1,(\rU,\rV)}) \approx 4a \cdot (d + \ln(1/\delta))/|\rU|$. Consider now some hypothesis $h_{\rS'}$ obtained by running ERM on a sub-sample $\rS'$ produced by the recursive call \textit{Sub-Sample}($\rU_0,\rV \cup \rU_1 \cup \rU_3$). The key observation and property of the \textit{Sub-Sample} algorithm, is that \emph{all} hypotheses in the majority vote $f_{1,(\rU,\rV)}$ have been trained on sub-samples that \emph{exclude} all of $\rU_1$. This means that the samples $\rU_1$ are independent of $f_{1,(\rU,\rV)}$. When $\Loss_\Dist(f_{1,(\rU,\rV)}) \approx 4a \cdot (d + \ln(1/\delta))/|\rU|$, we will now see about $|\rU_1|4a \cdot (d + \ln(1/\delta))/|\rU| = a \cdot (d + \ln(1/\delta))$ samples $(\rx,c(\rx))$ in $\rU_1$ for which $f_{1,(\rU,\rV)}(\rx) \neq c(\rx)$. The observation is that, conditioned on $f_{1,(\rU,\rV)}$, these samples are i.i.d. from the conditional distribution $\Dist( \cdot \mid f_{1,(\rU,\rV)} \textrm{ errs})$. The second key observation is that $h_{\rS'}$ is obtained by ERM on a sub-sample $\rS'$ that \emph{includes} all of $\rU_1$ (we add $\rU_1$ to $\rV$ in both of the other recursive calls). Moreover, since we are in the realizable setting, we have $h_{\rS'}(\rx) = c(\rx)$ for every $\rx \in \rU_1$. In particular, this holds for all the samples where $f_{1,(\rU,\rV)}(\rx) \neq c(\rx)$. The classic sample complexity bounds for proper PAC learning in the realizable setting then implies that $h_{\rS'}$ has $\Loss_{\Dist( \cdot \mid f_{1,(\rU,\rV)} \textrm{ errs})}(h_{\rS'}) = O((d + \ln(1/\delta))/(a \cdot (d + \ln(1/\delta)))) \leq 1/200$ for $a$ sufficiently large. Note that this is under the conditional distribution $\Dist(\cdot \mid f_{1,(\rU,\rV)} \textrm{ errs})$. That is, $h_{\rS'}$ rarely errs when $f_{1,(\rU,\rV)}$ errs. We thus get that $\Pr[f_{1,(\rU,\rV)} (\rx) \neq c(\rx) \wedge h_{\rS'}(\rx) \neq c(\rx)] = \Pr[f_{1,(\rU,\rV)}(\rx) \neq c(\rx)] \cdot \Pr[h_{\rS'}(\rx) \neq c(\rx) \mid f_{1,(\rU,\rV)}(\rx) \neq c(\rx)] \leq (a/50) \cdot (d + \ln(1/\delta))/|\rU|$. Since this holds for every $f_{i,(\rU,\rV)}$ and $h_{\rS'}$ from a recursive call $j \neq i$, we can now argue that $\Loss_\Dist(f_{(\rU,\rV)}) \ll a \cdot (d + \ln(1/\delta))/|\rU|$. To see this, note first that for $f_{(\rU,\rV)}$ to err on an $x \in \Xs$, it must be the case that at least one $f_{i,(\rU,\rV)}$ also errs. Even in this case, since one recursive call only contributes a third of the hypotheses in the majority vote $f_{(\rU,\rV)}$, there must be many hypotheses $h_{\rS'}$ (at least a $1/2-1/3 = 1/6$ fraction of all hypotheses) trained from sub-samples $\rS'$ produced by the recursive calls $j \neq i$ that also err on $x$. But we have just argued that it is very unlikely that both $f_{i,(\rU,\rV)}$ and such an $h_{\rS'}$ err at the same time. Formalizing this intuition completes the inductive proof.

\paragraph{Bagging Analysis.}
We now turn to presenting the key ideas in our proof of Theorem~\ref{thm:mainintro}, i.e. that bagging is an optimal PAC learner. Along the way, we also discuss the issues we encounter towards establishing the result. Recall that in bagging with a training set $\rS  \sim \Dist_c^m$, we randomly sub-sample $t$ bootstrap samples $\rS_1,\dots,\rS_t \subset \rS$ each consisting of $n$ i.i.d. samples with replacement from $\rS$. We then run ERM on each $\rS_i$ to produce hypotheses $h_{\rS_1},\dots,h_{\rS_t}$ and finally return the majority vote $f_{\rS_1,\dots,\rS_t}(x) = \sign(\sum_i h_{\rS_i}(x))$. It will be convenient for us to think of $f_{\rS_1,\dots,\rS_t}$ in a slightly different way. Concretely, we instead let $f_{\rS_1,\dots,\rS_t}(x) = (1/t)\sum_i h_{\rS_i}(x)$ be a \emph{voting classifier}. Then $\sign(f_{\rS_1,\dots,\rS_t}(x)) \neq c(x)$ if and only if $f_{\rS_1,\dots,\rS_t}(x)c(x) \leq 0$. We thus seek to bound $\Loss_\Dist(f_{\rS_1,\dots,\rS_t}) = \Pr_{\rx \sim \Dist}[f_{\rS_1,\dots,\rS_t} (\rx)c(\rx) \leq 0]$. The motivation for thinking about $f_{\rS_1,\dots,\rS_t}$ as a voting classifier, is that it allows us to re-use some of the ideas that appear in proving generalization bounds for AdaBoost~\cite{adaboost} and other voting classifiers.

We first observe that similarly to Hanneke's sub-sampling, if we look at just two hypotheses $h_{\rS_i}$ and $h_{\rS_j}$ with $i \neq j$, then $h_{\rS_i}$ is trained on a bootstrap sample $\rS_i$ leaving out a rather large portion of $\rS$. Furthermore, $h_{\rS_j}$ will be trained on most of these left-out samples and thus one could again argue that it is unlikely that $h_{\rS_i}$ and $h_{\rS_j}$ err at the same time. Unfortunately, this line of argument fails when we start combining a non-constant number of hypotheses. In particular, with high probability over the bootstrap samples, the union of any set of $\ell$ bootstrap samples contains all but an $\exp(-\Omega(\ell))$-fraction of $\rS$. This leaves very few samples that are independent of the hypotheses trained on such $\ell$ bootstrap samples. Trying to repeat Hanneke's argument unfortunately requires $\Omega(m)$ independent samples towards the last steps of an inductive proof. In a nutshell, what saves Hanneke's construction is that a third all sub-samples together still leave out a quarter of the training data. For bagging, such a property is just not true if we have more than a constant number of bootstrap samples.

Abandoning the hope of directly applying Hanneke's line of reasoning, we instead start by relating the performance of bagging to that of a particular voting classifier that is deterministically determined from a training set $S$, i.e. we get rid of the bootstrap samples. To formalize this, we first introduce some notation. From a training set $S$ of $m$ samples $(x_1,c(x_1)),\dots,(x_m,c(x_m))$ and a vector of $n$ not necessarily distinct integers $I = (i_1,\dots,i_n) \in [m]^n$, let $S(I)$ denote the bootstrap sample $(x_{i_1},c(x_{i_1})),\dots,(x_{i_n},c(x_{i_n}))$. Then a random bootstrap sample $\rS_i$ from $S$ has the same distribution as if we draw $\rInd$ uniformly from $[m]^n$ and let $\rS_i = S(\rInd)$. Also, let $h_{S(I)} \in \Hyp$ denote the hypothesis resulting from running ERM on $S(I)$. Finally, for a list of $t$ vectors $B = (I_1,\dots,I_t) \in [m]^{n \times t}$, we let $f_{S,B} = (1/t)\sum_{i=1}^t h_{S(I_i)}$ denote the voting classifier produced by bagging with bootstrap samples $S(I_1),\dots,S(I_t)$. Using the notation $\Boots \sim [m]^{n \times t}$ to denote a uniform random $\Boots$ from $[m]^{n \times t}$ we thus have that the hypothesis produced by bagging on $S$ has the same distribution as $f_{S,\Boots}$.

Now consider the following voting classifier
\[
  g_S(x) := \frac{1}{m^n} \sum_{I \in [m]^n} h_{S(I)}(x).
\]
That is, $g_S$ is the voting classifier averaging the predictions over all $m^n$ possible bootstrap samples of $S$. Of course one would never compute $g_S$. Nonetheless, the performance of the random $f_{S,\Boots}$ with $\Boots \sim [m]^{n \times t}$ is closely related to that of $g_S$. To see this, we introduce the notion of \emph{margins} which are typically used in the study of generalization of voting classifiers, see e.g. the works~\cite{bartlett:margins, gao2013doubt, Larsen2022OptimalWT}. For a sample $x \in \Xs$ and voting classifier $f(x) = (1/t)\sum_{i=1}^t h_i(x)$, we say that $f$ has margin $f(x)c(x)$ on the sample $x$. Since each $h_i(x)$ is in $\{-1,1\}$, we have the margin is a number between $-1$ and $1$. Intuitively, $1$ represents that all the hypotheses $h_i$ agree on the label of $x$ and those predictions are correct. In general, a margin of $\gamma$ implies that an $\alpha$-fraction of the hypotheses $h_i$ are correct, where $\alpha - (1-\alpha) = \gamma \Rightarrow \alpha = 1/2+\gamma/2$. For a margin $0 \leq \gamma \leq 1$, define $\Loss_\Dist^\gamma(f) = \Pr_{\rx \sim \Dist}[f(\rx)c(\rx) \leq \gamma]$. That is, $\Loss_\Dist^\gamma(f)$ is the probability over a random sample $\rx$ from $\Dist$ that $f$ has margin at most $\gamma$ on $\rx$. We have $\Loss_\Dist(f) = \Loss^0_{\Dist}(f)$. With margins defined, we show that for every training set $S = \{(x_i,c(x_i))\}_{i=1}^m$, if $t = \Omega(\ln(m/\delta))$, then with probability $1-\delta$ over $\Boots \sim [m]^{n \times t}$, we have
\begin{eqnarray}
  \label{eq:tog}
  \Loss_\Dist(f_{S,\Boots}) \leq \Loss^{1/3}_\Dist(g_S) + 1/m.
\end{eqnarray}
What this gives us, is that it suffices to understand how often the voting classifier that averages over all possible bootstrap samples has margin at most $1/3$. To see why~\eqref{eq:tog} is true, notice that every hypothesis $h_{S(\rInd_i)}$ in $f_{S,\Boots} = (1/t)\sum_{i=1}^t h_{S(\rInd_i)}$ is uniform random among the hypotheses averaged by $g_S$. Hence for any $x \in \Xs$ where $g_S$ has margin more than $1/3$, we have $\E_{\rInd_i \sim [m]^n}[h_{S(\rInd_i)}(x)c(x)] > 1/3$. A Chernoff bound and independence of the bootstrap samples implies that $\Pr_{\Boots \sim [m]^{n \times t}}[f_{S,\Boots}(x)c(x) \leq 0] \leq \exp(-\Omega(t)) \leq \delta/m$. Using that this holds for every $x$ with $g_S(x)c(x) > 1/3$ establishes~\eqref{eq:tog}.

Our next step is to show that $\Loss^{1/3}_\Dist(g_{\rS})$ is small with high probability over $\rS \sim \Dist^m_c$. For this, our key idea is to create groups of bootstrap samples $\rS(I)$ with $I \in [m]^n$.  These groups have a structure similar to those produced by Hanneke's \textit{Sub-Sample} procedure. We remark that these groups are only for the sake of analysis and are not part of the bagging algorithm.

For simplicity, let us for now assume that bagging produced samples without replacement instead of with replacement. To indicate this, we slightly abuse notation and let $\binom{m}{n}$ denote all vectors $I \in [m]^n$ where all entries are distinct. Also, let us assume that $n$ precisely equals the number of samples in each sub-sample created by Hanneke's \textit{Sub-Sample} (technically, this is $n = m - \sum_{i=0}^{\lg_4 m -1} 4^i$). For a set $S = (x_1,c(x_1)),\dots,(x_m,c(x_m))$, let $\HanI$ denote the collection of all vectors $I \in \binom{m}{n}$ such that $S(I)$ is one of the sub-samples produced by \textit{Sub-Sample}($S,\emptyset$). Note that $\HanI$ only depends on $m$, not on $S$ itself. We now define \emph{buckets} $\Buc_i$ of vectors $I \in \binom{m}{n}$ (i.e. of vectors corresponding to bootstrap samples). For every permutation $\pi$ of the indices $1,\dots,m$, we create a bucket $\Buc_\pi$. We add a vector $I=(i_1,\dots,i_n) \in \binom{m}{n}$ to $\Buc_\pi$ if and only if $\pi(I) =(\pi(i_1),\dots,\pi(i_n))$ is in $\HanI$.

With these buckets defined, we now make several crucial observations. First, for any bucket $\Buc_\pi$, if $\rS \sim \Dist^m_c$, then the joint distribution of the bootstrap samples $\rS(I)$ with $I \in \Buc_\pi$ is precisely the same as the joint distribution of the sub-samples produced by \textit{Sub-Sample}($\rS,\emptyset$). This holds since permuting the samples in $\rS$ does not change their distribution. Hence for any bucket, Hanneke's analysis shows that the majority of hypotheses $h_{\rS(I)}$ with $I \in \Buc_\pi$ rarely errs. More precisely, if we let $f_{\rS,\pi} = (1/|\Buc_\pi|) \sum_{I \in \Buc_\pi} h_{\rS(I)}$ then $\Loss_\Dist(f_{\rS,\pi}) = O((d+\ln(1/\delta))/m)$ with probability $1-\delta$ over $\rS$. Here we need something slightly stronger, namely that $\Loss^{5/6}_\Dist(f_{\rS,\pi}) = O((d+\ln(1/\delta))/m)$. Assume for now that this holds.

Next observe that if $x \in \Xs$ has $g_S(x)c(x) \leq 1/3$ for a training set $S$, then at least one third of the hypotheses $h_{S(I)}$ with $I \in \binom{m}{n}$ err on $x$, i.e. $h_{S(\rInd)}$ for $\rInd$ uniform in $\binom{m}{n}$ errs on $x$ with probability at least $1/3$ (here we assume that $g_S$ averages over $h_{S(I)}$ with $I \in \binom{m}{n}$, i.e. sampling without replacement instead of with). Symmetry now implies that every $I \in \binom{m}{n}$ is included in equally many buckets and all buckets contain equally many vectors $I$. This observation implies that the uniform $\rInd \sim \binom{m}{n}$ has the same distribution as if we first sample a uniform random bucket $\rBuc$ and then sample a uniform random $\rInd$ from $\rBuc$. But then if $h_{S(\rInd)}$ errs on $x$ with probability at least $1/3$, it must be the case that a constant fraction of the buckets $\Buc_\pi$ have $f_{S,\pi}(x)c(x) \leq 5/6$. This intuitively gives us that with high probability over $\rS$, $\Loss^{1/3}_\Dist(g_{\rS})$ can only be a constant factor larger than $\Loss^{5/6}_\Dist(f_{\rS,\pi})$. But $\Loss^{5/6}_\Dist(f_{\rS,\pi})$ is $O((d+\ln(1/\delta))/m)$ with high probability, establishing the same thing for $\Loss^{1/3}_\Dist(g_\rS)$ as desired.

The above are the main ideas in the proof, though making the steps completely formal requires some care. Also, let us briefly comment on some of the assumptions made above. First, we assumed that Hanneke's proof could establish the stronger claim $\Loss^{5/6}_\Dist(f_{\rS,\pi}) = O((d+\ln(1/\delta))/m)$, not just that $\Loss_\Dist(f_{\rS,\pi}) = O((d+\ln(1/\delta))/m)$. To formally carry out this argument, we have to change \textit{Sub-Sample} to partition $U$ into $20$ subsets instead of $4$ and make $19$ recursive calls instead of $3$. Next, we assumed above that $\rInd$ was sampled without replacement. This was used to establish symmetry of the buckets. To handle sampling with replacement, let $\Distinct(I)$ denote the set of distinct values appearing in the entries of a vector $I \in [m]^n$. We now create a family of buckets similarly to above for every cardinality of $|\Distinct(I)|$. This gives symmetry among buckets corresponding to the same cardinality. We then show that these families of buckets may be combined by allowing a non-uniform sampling over the families. Finally, this step also requires us to handle bootstrap samples corresponding to $I$ with very small $|\Distinct(I)|$, e.g. as small as just $1$. Since such $I$ are extremely unlikely, this can be charged to the failure probability $\delta$.

In the next section, we proceed to give the formal proof of Theorem~\ref{thm:mainintro}.

\section{Bagging is an optimal PAC learner}
\label{sec:proof}
We are ready to give the formal proof of Theorem~\ref{thm:mainintro}. We have restated the theorem here in the notation introduced in Section~\ref{sec:overview}
\begin{theorem}
  \label{thm:main}
There is a universal constant $a > 0$ such that for every $0 <
\delta < 1$, every distribution $\Dist$ over an input domain $\Xs$, every concept
class $\Hyp \subseteq \Xs \to \{-1,1\}$ of VC-dimension $d$ and every
$c \in \Hyp$, if $t \geq
18 \ln(2m/\delta)$ and $0.02m \leq n \leq m$, then it holds with
probability at least $1-\delta$ over the random choice of a training
set $\rS \sim \Dist^m_c$ and $t$ bootstrap samples $\Boots \sim [m]^{n \times t}$ that
\[
    \Loss_\Dist(f_{\rS,\Boots}) \leq a \cdot \frac{d + \ln(1/\delta)}{m}.
  \]
\end{theorem}
From hereon, we let $\Dist$ be an arbitrary distribution over an input domain $\Xs$, let $0 < \delta< 1$ and let $\Hyp$ be an arbitrary concept class of some fixed VC-dimension $d \geq 1$ and let $c \in \Hyp$ be arbitrary. For a training set $S$ and a bootstrap sample $S(I)$ with $I \in [m]^n$, we let $h_{S(I)} \in \Hyp$ be the hypothesis returned by the ERM algorithm used for bagging.

To prove Theorem~\ref{thm:main}, and hence Theorem~\ref{thm:mainintro}, we start by relating the performance of $f_{S,\Boots}$ with $\Boots \sim [m]^{n \times t}$ to another voting classifier $g_S$ that is a deterministic function of $S$. Concretely, for any training set $S$ of $m$ samples, let $g_S$ denote the voting classifier:
\[
  g_S(x) = \frac{1}{m^n} \sum_{I \in [m]^n} h_{S(I)}(x).
\]
That is, $g_S$ is the voting classifier averaging the predictions over all $m^n$ possible bootstrap samples. Recalling the definition of margins from Section~\ref{sec:overview}, we show the following relation between $\Loss_\Dist(f_{S,\Boots})$ and $\Loss^{1/3}_\Dist(g_S)$:
\begin{lemma}
  \label{lem:lossG}
  For every $0 < \delta < 1$ and training set $S$ of $m$ samples, if $t \geq  18 \ln(m/\delta)$, then it holds with probability at least $1-\delta$ over $\Boots \sim [m]^{n \times t}$ that
  \[
    \Loss_\Dist(f_{S,\Boots}) \leq \Loss^{1/3}_\Dist(g_S) + 1/m.
  \]
\end{lemma}

\begin{proof}
  Consider a fixed training set $S$ of $m$ samples, a fixed set of bootstrap samples $B \in [m]^{n \times t}$ and an $\rx \sim \Dist$. Let $E$ denote the event that $g_S(\rx)c(\rx) \leq 1/3$ and $\bar{E}$ the complement event. Then
  \begin{eqnarray*}
    \Loss_\Dist(f_{S,B}) &=& \Pr_{\rx \sim \Dist}[E] \Pr_{\rx \sim \Dist(\cdot \mid E)}[f_{S,B}(\rx)c(\rx)\leq 0 ] +  \Pr_{\rx \sim \Dist}[\bar{E}] \Pr_{\rx \sim \Dist(\cdot \mid \bar{E})}[f_{S,B}(\rx)c(\rx)\leq 0 ] \\
                          &\leq& \Pr_{\rx \sim \Dist}[E]  + \Pr_{\rx \sim \Dist(\cdot \mid \bar{E})}[f_{S,B}(\rx)c(\rx)\leq 0 ] \\
    &=& \Loss_\Dist^{1/3}(g_S) + \Pr_{\rx \sim \Dist(\cdot \mid \bar{E})}[f_{S,B}(\rx) c(\rx)\leq 0].
  \end{eqnarray*}
By linearity of expectation, we have
\begin{eqnarray*}
  \E_{\Boots \sim [m]^{n \times t}}[ \Pr_{\rx \sim \Dist(\cdot \mid \bar{E})}[f_{S,\Boots}(\rx) c(\rx) \leq 0 ]] &=& 
  \E_{\rx \sim \Dist(\cdot \mid \bar{E})}[\Pr_{\Boots \sim [m]^{n \times t}}[f_{S,\Boots}(\rx) c(\rx) \leq 0 ]].
\end{eqnarray*}
Now observe that every $\rInd_i$ in $\Boots = (\rInd_1,\dots,\rInd_t)$ is a uniform random sample from $[m]^n$, thus for $x$ satisfying $g_S(x)c(x) > 1/3$, each $\rInd_i$ has $\E_{\rInd_i \sim [m]^n}[h_{S(\rInd_i)}(x) c(x)] > 1/3$. By independence, and using $h_{S(\rInd_i)}(x) \in \{-1,1\}$, it follows from Hoeffding's inequality that for every $x$ with $g_S(x)c(x) > 1/3$, we have $\Pr_{\Boots \sim [m]^{n \times t}}[f_{S,\Boots}(x)c(x) \leq 0] \leq \exp(-(1/3)^2t/2) = \exp(-t/18)$. Thus for $t \geq 18 \ln(m/\delta)$, we have $\E_{\Boots \sim [m]^{n \times t}}[ \Pr_{\rx \sim \Dist(\cdot \mid \bar{E})}[f_{S,\Boots}(\rx) c(\rx) \leq 0 ]] \leq \delta/m$. Markov's inequality implies that with probability at least $1-\delta$ over the choice of $\Boots$, we have $\Pr_{\rx \sim \Dist(\cdot \mid \bar{E})}[f_{S,\Boots}(\rx) c(\rx)  \leq 0] \leq 1/m$ and thus $\Loss_\Dist(f_{S,\Boots}) \leq \Loss^{1/3}_\Dist(g_S) + 1/m$.
\end{proof}

In light of Lemma~\ref{lem:lossG} it thus suffices to understand $\Loss^{1/3}(g_S)$ for the deterministic $g_S$, i.e. we need to understand how often a third of the hypotheses $h_{S(I)}$ with $I \in [m]^n$ mispredict the label of a sample $\rx \sim \Dist$. For this, we prove the following:
\begin{lemma}
  \label{lem:goodG}
  There is a universal constant $a>0$, such that for every $0 < \delta < 1$, it holds with probability at least $1-\delta$ over the choice of a training set $\rS \sim \Dist^m_c$ that
  \[
    \Loss^{1/3}_\Dist(g_\rS) \leq a \cdot \frac{d + \ln(1/\delta)}{m}
  \]
\end{lemma}

Before we prove Lemma~\ref{lem:goodG}, let us see that it suffices to finish the proof of Theorem~\ref{thm:main}:
\begin{proof}[Proof of Theorem~\ref{thm:main}]
Assume $t \geq 18 \ln(2m/\delta)$. Let $\rS \sim \Dist^m_c$ be a random training set of $m$ samples and $\Boots \sim [m]^{n \times t}$. Define the event $E_f$ that occurs if $\Loss_\Dist(f_{\rS,\Boots}) > \Loss^{1/3}_\Dist(g_\rS) + 1/m$ and define $E_g$ as the event that $\Loss^{1/3}_\Dist(g_\rS) > a(d + \ln(2/\delta))/m$, where $a>0$ is the constant from Lemma~\ref{lem:goodG}. By Lemma~\ref{lem:lossG} and the constraint on $t$, we have $\Pr[E_f] \leq  \delta/2$. By Lemma~\ref{lem:goodG}, we have $\Pr[E_g] \leq \delta/2$. By a union bound, it holds with probability at least $1-\delta$ that none of the events $E_f$ and $E_g$ occur. In this case, we have $\Loss_\Dist(f_{\rS,\Boots}) \leq \Loss^{1/3}_{\Dist}(g_\rS) + 1/m \leq  a \cdot \frac{d + \ln(2/\delta)}{m} + 1/m \leq (a+2)\frac{d + \ln(1/\delta)}{m}$ (using $\ln(2) \leq 1 \leq d$). We have thus proved Theorem~\ref{thm:main} with the constant $a+2$.
\end{proof}

\subsection{Inducing Structure}
What remains is thus to prove Lemma~\ref{lem:goodG}. To bound $\Loss^{1/3}(g_{\rS})$, we show that we can partition the hypotheses in $\{h_{\rS(I)} : I \in [m]^n\}$ into structured groups that are easier to analyse. Concretely, let $\Fam$ be a collection of \emph{buckets} $\Buc_0, \Buc_1,\dots,\Buc_N \subseteq [m]^n$. We think of each $I \in \Buc_i$ as specifying the indices of a bootstrap sample. A bucket is thus a collection of vectors of indices. 

The buckets $\Buc_i$ constituting $\Fam$ need not be disjoint. For a training set $S$ and a bucket $\Buc$, we let $g_{S(\Buc)}$ denote the voting classifier
\[
  g_{S(\Buc)}(x) = \frac{1}{|\Buc|}\sum_{I \in \Buc} h_{S(I)}(x).
\]
Together with the family $\Fam$, we have a probability distribution $\Pdist$ over the buckets $\Buc_0,\dots,\Buc_N$. Using $\Pdist(i)$ to denote $\Pr_{\rBuc \sim \Pdist}[\rBuc = \Buc_i]$, we thus have $\sum_{i=0}^N \Pdist(i) = 1$ and $\Pdist(i) \geq 0$ for all $i$. We call the pair $(\Fam,\Pdist)$ $\eps$-\emph{representative} if the following holds:
\begin{enumerate}
\item If we first draw a bucket $\rBuc$ according to the distribution $\Pdist$ and then draw a uniform  $\rInd \in \rBuc$, then $\rInd$ is uniform random in $[m]^n$.
\item The probability $\Pdist(0)$ is at most $\eps$.
\end{enumerate}

The following lemma shows that if we have an $\eps$-representative pair $(\Fam,\Pdist)$, then it suffices to show that each of the voting classifiers $g_{\rS(\Buc_i)}$ with $i \neq 0$ perform well on a random training set $\rS \sim \Dist_c^m$
\begin{lemma}
  \label{lem:structure}
  Let $(\Fam,\Pdist)$ be $1/6$-representative. Then for any $\tau \geq 0$, we have
  \[
    \Pr_{\rS \sim \Dist_c^m}[\Loss^{1/3}_\Dist(g_\rS) > 24 \cdot \tau] \leq 23 \cdot \max_{i \neq 0} \Pr_{\rS \sim \Dist_c^m}[ \Loss_\Dist^{5/6}(g_{
      \rS(\Buc_i)}) > \tau].
  \]
\end{lemma}

\begin{proof}[Proof of Lemma~\ref{lem:structure}]
Consider an arbitrary training set $S$ of $m$ samples and an $x \in \Xs$ for which $g_S(x)c(x) \leq 1/3$. If we let $\rInd$ be uniform random in $[m]^n$, then $\Pr_{\rInd \sim [m]^n}[h_{S(\rInd)}(x) \neq c(x)] \geq 1/3$. At the same time, since $(\Fam,\Pdist)$ is $1/6$-representative, $\rInd$ has the same distribution as if we first draw a bucket $\rBuc$ according to $\Pdist$, then sample a uniform $\rInd$ in $\rBuc$. Using $\rInd \sim \rBuc$ to denote that $\rInd$ is drawn uniformly from $\rBuc$, we thus have $\E_{\rBuc \sim \Pdist}[\Pr_{\rInd \sim \rBuc}[h_{S(\rInd)}(x)\neq c(x)]] \geq 1/3$. On the other hand,
  \begin{eqnarray*}
    \E_{\rBuc \sim \Pdist}[\Pr_{\rInd \sim \rBuc}[h_{S(\rInd)}(x)\neq c(x)]] &\leq& \\
    \Pr_{\rBuc \sim \Pdist}[\rBuc = \Buc_0] + \Pr_{\rBuc \sim \Pdist}[g_{S(\rBuc)}(x)c(x) \leq 5/6 \wedge \rBuc \neq \Buc_0] + \E_{\rBuc \sim \Pdist}[\Pr_{\rInd \sim \rBuc}[h_{S(\rInd)}(x)\neq c(x)] \mid g_{S(\rBuc)}(x)c(x) > 5/6]&\leq& \\
    \Pdist(0) + \Pr_{\rBuc \sim \Pdist}[g_{S(\rBuc)}(x)c(x) \leq 5/6 \wedge \rBuc \neq \Buc_0] + ( 1/2-5/12) &\leq& \\
    1/6 + \Pr_{\rBuc \sim \Pdist}[g_{S(\rBuc)}(x)c(x) \leq 5/6 \wedge \rBuc \neq \Buc_0] + 1/12.
  \end{eqnarray*}
  We thus have $\Pr_{\rBuc \sim \Pdist}[g_{S(\rBuc)}(x)c(x) \leq 5/6 \wedge \rBuc \neq \Buc_0] \geq 1/3 - 1/6 - 1/12 = 1/12$ for such $S,x$.

For a fixed $S$ and an $\rx \sim \Dist$, let $E$ denote the event that $g_S(\rx)c(\rx) \leq 1/3$. Then the above implies that
  \begin{eqnarray*}
    1/12 &\leq& \E_{\rx \sim \Dist(\cdot \mid E)}[\Pr_{\rBuc \sim \Pdist}[g_{S(\rBuc)}(\rx)c(\rx) \leq 5/6  \wedge \rBuc \neq \Buc_0]] \\
         &=&
             \E_{\rBuc \sim \Pdist}[\Pr_{\rx \sim \Dist(\cdot \mid E)}[g_{S(\rBuc)}(\rx)c(\rx) \leq 5/6  \wedge \rBuc \neq \Buc_0]] 
  \end{eqnarray*}
For short, let $Z(\Buc)$ take the value $\Pr_{\rx \sim \Dist(\cdot \mid E)}[g_{S(\rBuc)}(\rx)c(\rx) \leq 5/6  \wedge \Buc \neq \Buc_0]$ (which is trivially $0$ for $\Buc=\Buc_0$). Then the above is $\E_{\rBuc \sim \Pdist}[Z(\rBuc)] \geq 1/12$. Moreover, $1-Z(\rBuc)$ is non-negative and $\E_{\rBuc \sim \Pdist}[1-Z(\rBuc)] \leq 11/12$. Hence by Markov's inequality, we have $\Pr_{\rBuc \sim \Pdist}[Z(\rBuc) \leq 1/24] = \Pr_{\rBuc \sim \Pdist}[1-Z(\rBuc) \geq 23/24] \leq (11/12)/(23/24) = 22/23$. Thus $\Pr_{\rBuc \sim \Pdist}[Z(\rBuc) \geq 1/24] \geq 1/23$. Now observe that if $\Buc$ satisfies $Z(\Buc) \geq 1/24$, then $\Buc \neq \Buc_0$ and 
  \begin{eqnarray*}
    \Loss^{5/6}_\Dist(g_{S(\Buc)}) &=& \\
    \Pr_{\rx \sim \Dist}[g_{S(\Buc)}(\rx)c(\rx) \leq 5/6 ] &\geq& \\
    \Pr_{\rx \sim \Dist}[E] \cdot \Pr_{\rx \sim \Dist(\cdot \mid E)}[g_{S(\Buc)}(\rx)c(\rx) \leq 5/6 ]  &\geq& \\
    \Pr_{\rx \sim \Dist}[E] /24 &=& \\
    \Loss^{1/3}_\Dist(g_S)/24.
  \end{eqnarray*}
  We therefore conclude that
  \begin{eqnarray*}
    \Pr_{\rBuc \sim \Pdist}[ \Loss_\Dist^{5/6}(g_{S(\rBuc)}) \geq \Loss^{1/3}_\Dist(g_S)/24 \wedge \rBuc \neq \Buc_0] &\geq& \\
\Pr_{\rBuc \sim \Pdist}[Z(\rBuc) \geq 1/24] &\geq& \\
    1/23.
  \end{eqnarray*}
  This implies that for any $\tau \geq 0$ and any $S$ with $\Loss^{1/3}_\Dist(g_S) > 24 \cdot \tau$, we also have $\Pr_{\rBuc \sim \Pdist}[ \Loss_\Dist^{5/6}(g_{S(\rBuc)}) > \tau \wedge \rBuc \neq \Buc_0] \geq 1/23$. Therefore, by Markov's inequality
  \begin{eqnarray*}
    \Pr_{\rS \sim \Dist_c^m}[\Loss^{1/3}_\Dist(g_\rS) > 24 \cdot \tau] &\leq& \\
    \Pr_{\rS \sim \Dist_c^m}[\Pr_{\rBuc \sim \Pdist}[ \Loss_\Dist^{5/6}(g_{\rS(\rBuc)}) > \tau \wedge \rBuc \neq \Buc_0] \geq 1/23] &\leq& \\
    23 \cdot \E_{\rS \sim \Dist_c^m}[\Pr_{\rBuc \sim \Pdist}[ \Loss_\Dist^{5/6}(g_{\rS(\rBuc)}) > \tau \wedge \rBuc \neq \Buc_0]] &=& \\
    23 \cdot \E_{\rBuc \sim \Pdist}[\Pr_{\rS \sim \Dist_c^m}[ \Loss_\Dist^{5/6}(g_{\rS(\rBuc)}) > \tau \wedge \rBuc \neq \Buc_0]]  &\leq& \\
    23 \cdot \max_{i \neq 0} \Pr_{\rS \sim \Dist_c^m}[ \Loss_\Dist^{5/6}(g_{\rS(\Buc_i)}) > \tau].
  \end{eqnarray*}
\end{proof}

We next show that it is possible to construct a representative pair in which all buckets except $\Buc_0$ behave nicely

\begin{lemma}
  \label{lem:goodRep}
  There is a universal constant $a>0$, such that for any $m$ and $n$ with $0.02m \leq n \leq m$, there is a $1/6$-representative pair $(\Fam = \Buc_0,\dots,\Buc_N,\Pdist)$ satisfying that for every $0 < \delta < 1$, it holds that
  \[
    \max_{i \neq 0} \Pr_{\rS \sim \Dist_c^m}[\Loss^{5/6}_\Dist(g_{\rS(\Buc_i)}) > a(d + \ln(1/\delta))/m] < \delta.
  \]
\end{lemma}

Before proving Lemma~\ref{lem:goodRep}, let us see that it suffices to complete the proof of Lemma~\ref{lem:goodG}.
\begin{proof}[Proof of Lemma~\ref{lem:goodG}]
Let $(\Fam,\Pdist)$ be as in Lemma~\ref{lem:goodRep} and let $a$ be the constant in Lemma~\ref{lem:goodRep}. For any $0 < \delta < 1$, Lemma~\ref{lem:structure} givet us that
  \begin{eqnarray*}
    \Pr_{\rS \sim \Dist_c^m}[\Loss^{1/3}_\Dist(g_\rS) > 76 a(d + \ln(1/\delta))/m] &\leq& \Pr_{\rS \sim \Dist_c^m}[\Loss^{1/3}_\Dist(g_\rS) > 24 a(d + \ln(23/\delta))/m] \\
    &\leq& 23 \cdot \max_{i \neq 0} \Pr_{\rS \sim \Dist_c^m}[\Loss^{5/6}_\Dist(g_{\rS(\Buc_i)}) > a(d + \ln(23/\delta))/m] \\
                                                                                 &\leq& 23 \cdot (\delta/23) \\
    &=&\delta.
  \end{eqnarray*}
Thus we have proved Lemma~\ref{lem:goodG} with the constant $76 a$.
\end{proof}
In the next section, we define the pair $(\Fam,\Pdist)$ having the properties claimed in Lemma~\ref{lem:goodRep}.

\subsection{Representative Pair}
We now define a $1/6$-representative pair $(\Fam,\Pdist)$ with $\Fam = \Buc_0,\dots,\Buc_N$ for some $N$, such that $(\Fam,\Pdist)$ has the properties claimed in Lemma~\ref{lem:goodRep}.

We construct a number of buckets for each possible number of distinct values among the entries of a vector $I \in [m]^n$. Recall that we assume $0.02m \leq n \leq m$.

\paragraph{Bucket $\Buc_0$.}
We use the special bucket $\Buc_0$ to handle all vectors $I \in [m]^n$ with fewer than $0.01m$ distinct values among their entries as well as those $I$ with more than $0.9m$ distinct values. That is, if we let $\Distinct(I) \subseteq [m]$ denote the set of distinct values occurring in $I$, then $\Buc_0 = \{ I \in [m]^n : |\Distinct(I)| \notin [0.01m, 0.9m] \}$.

\paragraph{Remaining Buckets.}
For every integer $d \in \{\lceil 0.01m \rceil,\dots \lfloor 0.9m \rfloor \}$, we construct a number of buckets $\Buc_{d,1},\dots,\Buc_{d,N_d}$. The final set of buckets $\Fam$ is simply $\{\Buc_0\} \cup \left(\bigcup_d \bigcup_{i \in [N_d]} \{\Buc_{d,i}\}\right)$.

Let $\ell = 20^j$ be the largest power of $20$ such that $\ell \leq 0.01m$. For every $d$, we construct a bucket for every list $L$ of $\ell$ distinct indices in $[m]$.  We thus have $N_d=m(m-1)\cdots (m-\ell+1) = m!/(m-\ell)!$. We only add vectors $I \in [m]^n$ with $|\Distinct(I)|=d$ to these buckets. Given a list $L = i_1,\dots,i_\ell$ of distinct indices, invoking Algorithm~\ref{alg:buckets} as \textit{BuildBucket}$_d$($L,\emptyset,\emptyset$) constructs the bucket corresponding to $L$. 

\begin{algorithm}
  \DontPrintSemicolon
  \KwIn{List $L$ of $\ell=20^j$ distinct indices $i_1,\dots,i_\ell$ for an integer $j \geq 0$\\
    Set of indices $U \subseteq [m]$ that must be contained in $I$\\
    Set of indices $V \subseteq [m]$ that must be disjoint from $I$
  }
  \KwResult{Bucket $\Buc$ containing vectors $I \in [m]^n$}
  \eIf{$j = 0$}{
    $U \gets U \cup \{i_1\}$\\
    \Return{$\{ I \in [m]^n : |\Distinct(I)|=d \wedge U \subseteq \Distinct(I) \wedge V \cap \Distinct(I) = \emptyset \}$}
    }
    {
      Partition $L$ into $20$ consecutive groups $L_0,\dots,L_{19}$ of $20^{j-1}$ indices each.\\
      \Return{$\bigcup_{i=1}^{19}$BuildBucket$_d(L_0, U \cup (\bigcup_{j \in \{1,\dots,19\} \setminus \{i\}} L_j),V  \cup L_i)$} 
    }
  \caption{\textit{BuildBucket}$_d$($L,U,V$)}\label{alg:buckets}
\end{algorithm}

Examining the algorithm, we observe that the recursive procedure computes disjoint sets $U \subset L$ and $V \subset L$. Whenever the recursion bottoms out and we return a subset of $[m]^n$ in step 3 of Algorithm~\ref{alg:buckets}, we have $U \cup V = L$ and we output all $I \in [m]^n$ with  $|\Distinct(I)|=d$, $U \subseteq \Distinct(I)$ and $V \cap \Distinct(I)=\emptyset$. Since $|L|=\ell \leq 0.01m$ and $d \in \{\lceil 0.01m \rceil,\dots \lfloor 0.9m \rfloor \}$, it follows that the output is always non-empty as $|[m] \setminus V| \geq 0.99m \geq d$ and $d \geq 0.01m \geq |U|$.

\paragraph{Probability Distribution.}
We next argue that we can define a probability distribution $\Pdist$ over the buckets $\Buc_0,\dots,\Buc_N$ defined above, such that if we draw a bucket $\rBuc$ according to $\Pdist$ and then a uniform $\rInd$ from $\rBuc$, then the distribution of $\rInd$ is uniform in $[m]^n$ and $\Pdist(0) \leq 1/6$.

Consider first the set of buckets $\Buc_{d,1},\dots,\Buc_{d,N_d}$ added for a particular value of $d \in \{\lceil 0.01m \rceil,\dots \lfloor 0.9m \rfloor \}$. These buckets only contain vectors $I \in [m]^n$ with $|\Distinct(I)|=d$. By symmetry, since we have a bucket for every list $L$ of distinct indices $i_1,\dots,i_\ell$, we have that every $I$ with $|\Distinct(I)|=d$ is contained in equally many buckets $\Buc_{d,i}$, and every bucket contains equally many vectors $I$. Thus if we draw a uniform random bucket $\rBuc$ among $\Buc_{d,1},\dots,\Buc_{d,N_d}$ and then a uniform random $\rInd$ in $\rBuc$, then the distribution of $\rInd$ is uniform among all $I$ with $|\Distinct(I)|=d$.

For every such bucket $\Buc_{d,i}$, we now define the corresponding probability under $\Pdist$ to be $\Pr_{\rInd \sim [m]^n}[|\Distinct(\rInd)|=d] N_d^{-1}$. Similarly, we define $\Pdist(0)$ to equal $\Pr_{\rInd \sim [m]^n}[|\Distinct(\rInd)| \notin  [0.01m, 0.9m]]$. With these probabilities, we have that if we first draw a bucket $\rBuc$ according to $\Pdist$ and then draw a uniform random $\rInd$ from $\rBuc$, then the distribution of $\rInd$ is uniform in $[m]^n$ as required.

\begin{lemma}
  \label{lem:sizesample}
  For the above $(\Fam,\Pdist)$, if $m \geq 100$, then it holds that $\Pdist(0) \leq 1/6$.
\end{lemma}
The proof is a standard counting argument for sampling with replacement and can be found in the appendix.

In the following section, we prove that the pair $(\Fam,\Pdist)$ constructed above satisfies the properties in Lemma~\ref{lem:goodRep} (if $m < 100$, we do not need to prove anything as $\Pr_{\rS \sim \Dist^m_c}[\Loss^{5/6}_\Dist(g_{\rS(\Buc_i)}) > a(d + \ln(1/\delta))/100]$ is trivially $0$ provided $a \geq 100$).

\subsection{Performance in a Bucket}
What remains is to understand the behaviour of a voting classifier constructed from the indices in a bucket, i.e. we complete the proof of Lemma~\ref{lem:goodRep}. As mention above, we may assume $m \geq 100$ as otherwise the claim in Lemma~\ref{lem:goodRep} is trivially true for $a \geq 100$. The proof follows that of Hanneke rather closely.

For the $1/6$-representative pair $(\Fam=\Buc_0,\dots,\Buc_N,\Pdist)$ defined in the previous section, we consider the buckets $\Buc$ with $\Buc \neq \Buc_0$. By construction, any such bucket was obtained from a choice of $d = |\Distinct(I)|$ and a list $L=i_1,\dots,i_\ell$ of $\ell = 20^j$ distinct indices. So consider a bucket $\Buc$ corresponding to some $d \in \{\lceil 0.01m \rceil, \dots,\lfloor 0.9 m\rfloor \}$ and a list $L$. We show that for any distribution $\Dist$, with high probability over the choice of $\rS \sim \Dist_c^m$, we have $\Loss^{5/6}_\Dist(g_{\rS(\Buc)}) \leq a(d + \ln(1/\delta))/m$.

Recall that from $L$ and $d$, we construct $\Buc$ by invoking Algorithm~\ref{alg:buckets} as \textit{BuildBucket}$_d$($L,\emptyset,\emptyset$). Now consider the recursion tree $\Tree$ corresponding to the invocation of \textit{BuildBucket}$_d$($L,\emptyset,\emptyset$). The root corresponds to the initial call \textit{BuildBucket}$_d$($L,\emptyset,\emptyset$). Each internal node has $19$ children $w_1,\dots,w_{19}$, one for each of the $19$ recursive calls \textit{BuildBucket}$_d$($L_0, U \cup (\bigcup_{j \in \{1,\dots,19\} \setminus \{i\} } L_j ),V \cup L_i$). We associate with each node $v \in \Tree$ the sets $L_v, U_v$ and $V_v$ corresponding to the arguments given as parameters to the corresponding call to \textit{BuildBucket}. Similarly, for every node $v \in \Tree$, we define $\Sub_v$ as the collection of all $I \in [m]^n$ added to the output $\Buc$ during recursive calls starting at $v$. That is, if $v$ is a leaf, then $\Sub_v = \{I \in [m]^n  : \Distinct(I)=d \wedge (U_v \cup \{i_1\}) \subseteq \Distinct(I) \wedge V_v \cap \Distinct(I)=\emptyset\}$ and if $v$ is internal with children $w_1,\dots,w_{19}$, then $\Sub_v = \cup_{i=1}^{19} \Sub_{w_i}$.

With these definitions, we consider the intermediate voting classifiers obtained from the sets $\Sub_v$:
\[
  g_{S(\Sub_v)}(x) = \frac{1}{|\Sub_v|} \sum_{I \in \Sub_v} h_{S(I)}(x).
\]
Importantly, if $r$ is the root of $\Tree$, then $g_{S(\Sub_r)} = g_{S(\Buc)}$.

For a node $v \in \Tree$, we say that its height is the distance from $v$ to the leaves in its subtree. Hence a leaf has height $0$ and the root has height $j$. We prove the following
\begin{lemma}
  \label{lem:recursive}
  There is a universal constant $a>0$ such that for any distribution $\Dist$, any $0 < \delta < 1$, it holds for every node $v \in \Tree$ of height $\kappa \in \{0,\dots,j\}$ that
  \[
    \Pr_{\rS \sim \Dist_c^m}[\Loss^{5/6}_\Dist(g_{\rS(\Sub_v)}) > a (d + \ln(1/\delta))/20^\kappa] < \delta.
  \]
\end{lemma}
Using Lemma~\ref{lem:recursive}, we now prove Lemma~\ref{lem:goodRep}:
\begin{proof}[Proof of Lemma~\ref{lem:goodRep}]
Let $a$ be the constant in Lemma~\ref{lem:recursive}. Using that $g_{S(\Sub_r)} = g_{S(\Buc)}$ when $r$ is the root of $\Tree$ shows that $\Pr_{\rS \sim \Dist_c^m}[\Loss^{5/6}_\Dist(g_{\rS(\Buc)}) > a(d + \ln(1/\delta))/20^j] < \delta$. Using that $20^j$ was defined to be the largest power of $20$ such that $20^j \leq 0.01m$, we have $20^j \geq m/2000$. Thus $\Pr_{\rS \sim \Dist_c^m}[\Loss^{5/6}_\Dist(g_{\rS(\Buc)}) > 2000 \cdot a(d + \ln(1/\delta))/m] < \delta$. Since this holds for all buckets $\Buc \neq \Buc_0$, we have thus proved Lemma~\ref{lem:goodRep} with the constant $2000 a$.
\end{proof}

What remains is to prove Lemma~\ref{lem:recursive}. We prove lemma by induction in $\kappa$. For the base case $\kappa=0$, the proof is trivial as we always have  $\Loss^{5/6}_\Dist(g_{\rS(\Sub_v)}) \leq 1 < a (d + \ln(1/\delta))$ provided $a \geq 1$.

For the inductive step, assume the lemma holds for all $\kappa' < \kappa$. Now consider a fixed value of $0 < \delta < 1$, a distribution $\Dist$ and a fixed node $v \in \Tree$ of height $\kappa$ with children $w_1,\dots,w_{19}$. To complete the inductive step, we define a number of undesirable properties of a training set $S$ and show that a random $\rS$ rarely has any of these properties. The first property we define is the following
\begin{itemize}
\item For $i=1,\dots,19$, let $\sSet_i$ denote the set of all $S$ for which $\Loss^{5/6}_\Dist(g_{S(\Sub_{w_i})}) > a (d + \ln(57/\delta))/20^{\kappa-1}$.
\end{itemize}
Here we trivially have
\begin{lemma}
  \label{lem:p1}
For any $i$, we have $\Pr_{\rS \sim \Dist_c^m}[\rS \in \sSet_i] \leq \delta/57$.
\end{lemma}
\begin{proof}
  The lemma follows immediately from the induction hypothesis.
\end{proof}

Next, for a child $w_i$ and a training set $S$, let $E_{S,i} \subseteq \Xs \times \{-1,1\}$ denote the subset of all $(x,c(x))$ with $x \in \Xs$ for which $g_{S(\Sub_{w_i})}(x)c(x) \leq 5/6$. Also for the node $v$ let $L_0,\dots,L_{19}$ be the partitioning of $L_v$ by \textit{BuildBucket}. The second property is
\begin{itemize}
\item For $i=1,\dots,19$, let $\sSet_i'$ denote the set of all training sets $S$ for which $\Loss^{5/6}_\Dist(g_{S(\Sub_{w_i})}) \geq (a/627)(d + \ln(1/\delta))/20^\kappa$ and $|S(L_{i}) \cap E_{S.i}| < 10^{9}(d+\ln(57/\delta))$.
\end{itemize}
Note in the above that $L_{i}$ is a set of distinct indices and $S(L_{i})$ is the subset of samples in $S$ indexed by $L_{i}$. The property intuitively states that $g_{S(\Sub_{w_i})}$ often has a margin less than $5/6$, yet we see few such samples in $S(L_{i})$.

For the third and last property, define $\Dist(\cdot \mid E_{S,i})$ for a set of samples $S$, as the distribution of an $\rx$ drawn from $\Dist$, conditioned on $(\rx,c(\rx)) \in E_{S,i}$. The final property is
\begin{itemize}
  \item For $i=1,\dots,19$, let $\sSet_i''$ denote the set of all training sets $S$ for which $|S(L_{i}) \cap E_{S,i}| \geq 10^{9}(d+\ln(57/\delta))$ and there is a hypothesis $h \in \Hyp$ such that $h(x)=c(x)$ for all $(x,c(x)) \in S(L_{i}) \cap E_{S,i}$ and $\Loss_{\Dist(\cdot \mid E_{S,i})}(h) > 1/62700$.
\end{itemize}

We first argue that every training set that does not lie in any of the sets $\sSet_i,\sSet_i',\sSet_i''$ behaves well
\begin{lemma}
  \label{lem:goodSsuffice}
  Let $S$ satisfy $S \notin \bigcup_i \{\sSet_i \cup \sSet_i' \cup \sSet_i''\}$. Then
  \[
    \Loss^{5/6}_\Dist(g_{S(\Sub_v)}) \leq a(d + \ln(1/\delta))/20^\kappa.
  \]
\end{lemma}
Combining Lemma~\ref{lem:goodSsuffice} with the following completes the inductive proof of Lemma~\ref{lem:recursive}
\begin{lemma}
  \label{lem:goodSlikely}
  Let $\rS \sim \Dist_c^m$. Then if the constant $a>0$ is large enough, it holds that $\Pr_{\rS \sim \Dist_c^m}[\rS \notin \bigcup_i \{\sSet_i \cup \sSet_i' \cup \sSet_i''\}] \geq 1-\delta$.
\end{lemma}
Thus what remains is to prove Lemma~\ref{lem:goodSsuffice} and Lemma~\ref{lem:goodSlikely}. We start with the proof of Lemma~\ref{lem:goodSsuffice}.

\subsection{Well-Behaved $S$}
\begin{proof}[Proof of Lemma~\ref{lem:goodSsuffice}]
  Let $S$ be a training set of $m$ samples satisfying $\forall i : S \notin  (\sSet_i \cup \sSet_i' \cup \sSet_i'')$. Consider any child $w_i$ of $v$ and a vector $I \in \Sub_{w_j}$ for a child $w_j \neq w_i$. We focus on the hypothesis $h_{S(I)}$ produced by ERM on $S(I)$. Since we are in the realizable case, we know that $h_{S(I)}$ correctly predicts the label of every sample it is trained on, i.e. $h(x)=c(x)$ for all $(x,c(x)) \in S(I)$. Furthermore, by the recursive procedure \textit{BuildBucket}$_d$, we see that all of $L_i$ is added to $U_{w_j}$ in the recursive call corresponding to the child $w_j$. Hence every sample in $S(L_{i})$ is included in $S(I)$. We now aim to show that it is unlikely that both $h_{S(I)}$ mispredict the label of an $x$ while $g_{S(\Sub_{w_i})}$ also has a margin of no more than $5/6$ on $x$
  \begin{eqnarray}
    \label{eq:combine}
    \Pr_{\rx \sim \Dist}[g_{S(\Sub_{w_i})}(\rx)c(\rx) \leq 5/6 \wedge h_{S(I)}(\rx) \neq c(\rx)] \leq (a/627)(d+\ln(1/\delta))/20^\kappa.
  \end{eqnarray}
  We consider two cases. In the first case, we have $\Loss^{5/6}_\Dist(g_{S(\Sub_{w_i})}) < (a/627)(d+ \ln(1/\delta))/20^\kappa$. Here the conclusion~\eqref{eq:combine} follows trivially as
  \begin{eqnarray*}
    \Pr_{\rx \sim \Dist}[g_{S(\Sub_{w_i})}(\rx)c(\rx) \leq 5/6 \wedge h(\rx) \neq c(\rx)] &\leq& \\
    \Pr_{\rx \sim \Dist}[g_{S(\Sub_{w_i})}(\rx)c(\rx) \leq 5/6] &=& \\
    \Loss^{5/6}_\Dist [g_{S(\Sub_{w_i})}].
  \end{eqnarray*}
  In the second case, we have $\Loss^{5/6}_\Dist(g_{S(\Sub_{w_i})}) \geq (a/627)(d+ \ln(1/\delta))/20^\kappa$.  Since $S \notin \sSet_i'$, this implies $|S(L_{i}) \cap E_{S,i} | \geq 10^{9}(d + \ln(57/\delta))$. Since $S \notin \sSet_i''$, this gives us that $\Loss_{\Dist( \cdot \mid E_{S,i})}(h_{S(I)}) \leq 1/62700$. Using that $S \notin \sSet_i$, we at the same time have $\Loss^{5/6}_\Dist(g_{S(\Sub_{w_i})}) \leq a (d + \ln(57/\delta))/20^{\kappa-1}$. We therefore have
  \begin{eqnarray*}
    \Pr_{\rx \sim \Dist}[g_{S(\Sub_{w_i})}(\rx)c(\rx) \leq 5/6 \wedge h_{S(I)}(\rx) \neq c(\rx)] &=& \\
    \Loss_\Dist^{5/6}(g_{S(\Sub_{w_i})}) \cdot \Loss_{\Dist(\cdot \mid E_{S,i})}(h_{S(I)}) &\leq& \\
    (a(d + \ln(57/\delta))/20^{\kappa-1}) \cdot (1/62700) &\leq& \\
    (100a(d + \ln(1/\delta))/20^{\kappa}) \cdot (1/62700) &\leq& \\
    (a/627)(d + \ln(1/\delta))/20^{\kappa}).
  \end{eqnarray*}
  We have thus established~\eqref{eq:combine}. Using this, we now show that $\Loss^{5/6}_\Dist(g_{S(\Sub_v)}) \leq a(d+\ln(1/\delta))/20^\kappa$. Define $E_S$ as the set of $(x,c(x))$ with $x \in \Xs$ for which $g_{S(\Sub_v)}(x)c(x) \leq 5/6$. For any such $(x,c(x)) \in E_S$, consider drawing two vectors $\rInd_1,\rInd_2$ uniformly and independently from $\Sub_v$. Since each $\Sub_{w_i}$ contains equally many vectors $I$, we have that $\rInd_1$ and $\rInd_2$ have the same distribution as if we first pick a uniform random child $w_\ri$ and then a uniform random $\rInd$ from $\Sub_{w_\ri}$. Let $\ri_1$ denote the index of the child containing $\rInd_1$ and let $\ri_2$ denote the index of the child containing $\rInd_2$.

  We first observe that since $(x,c(x)) \in E_S$, there must be a child index $z$ for which $g_{S(\Sub_{w_{z}})}(x)c(x) \leq 5/6$. Fix an arbitrary such $z(x)$ for every $(x,c(x)) \in E_S$. Since $(x,c(x)) \in E_S$, we have that at least an $\alpha$-fraction of the vectors $I$ in $\Sub_v$ have $h_{S(I)}(x) \neq c(x)$ where $\alpha$ satisfies $1-2\alpha \leq 5/6 \Rightarrow \alpha \geq 1/12$. Hence even outside $\Sub_{w_{z(x)}}$ there is still another $|\Sub_v|/12 - |\Sub_{w_{z(x)}}| = |\Sub_v|/12-|\Sub_v|/19 = (7/228)|\Sub_v| > |\Sub_v|/33$ vectors $I \in \Sub_v \setminus \Sub_{w_{z(x)}}$ for which $h_{S(I)}(x) \neq c(x)$. We therefore have for $(x,c(x)) \in E_S$ that
    \[
      \Pr_{\rInd_1,\rInd_2 \sim \Sub_{v}}[\ri_1 = z(x) \wedge \ri_2 \neq z(x) \wedge h_{S(\rInd_2)}(x) \neq c(x)] > (1/19) \cdot (1/33) = 1/627.
      \]
      It follows that 
      \[
        \Pr_{\rInd_1,\rInd_2 \sim \Sub_v, \rx \sim \Dist}[\ri_1 = z(\rx) \wedge \ri_2 \neq z(\rx) \wedge h_{S(\rInd_2)}(\rx) \neq c(\rx)] > \Pr_{\rx \sim \Dist}[(\rx,c(\rx)) \in E_S]/627.
      \]
      By averaging, there must exist a fixed choice of $I_1^*$ and $I_2^*$ such that
      \[
        \Pr_{\rx \sim \Dist}[i^*_1 = z(\rx) \wedge i^*_2 \neq z(\rx) \wedge h_{S(I^*_2)}(\rx) \neq c(\rx)] > \Pr_{\rx \sim \Dist}[(\rx,c(\rx)) \in E_S]/627.
      \]
      If $i^*_1=i^*_2$, then the left hand side is $0$ and we conclude $\Loss^{5/6}_\Dist(g_{S(\Sub_v)}) = \Pr_{\rx \sim \Dist}[(\rx,c(\rx)) \in E_S] = 0$. Otherwise, we have $i^*_1 \neq i^*_2$ and
      \[
        \Pr_{\rx \sim \Dist}[i^*_1 = z(\rx) \wedge i^*_2 \neq z(\rx) \wedge h_{S(I^*_2)}(\rx) \neq c(\rx)] \leq \Pr_{\rx \sim \Dist}[i^*_1=z(\rx) \wedge h_{S(I^*_2)}(\rx) \neq c(\rx)].
      \]
      Since $I^*_2$ comes from a child $i_2^* \neq i^*_1$, it follows from~\eqref{eq:combine} that
      \[
        \Pr_{\rx \sim \Dist}[i^*_1=z(\rx) \wedge h_{S(I^*_2)}(\rx) \neq c(\rx)] \leq (a/627)(d + \ln(1/\delta))/20^\kappa
      \]
      This finally implies
      \[
        \Loss^{5/6}_\Dist(g_{S(\Sub_v)}) = \Pr_{\rx \sim \Dist}[(\rx,c(\rx)) \in E_S] < 627(a/627)(d + \ln(1/\delta))/20^\kappa = a(d + \ln(1/\delta))/20^\kappa.
      \]
\end{proof}

\subsection{Most $S$ are Well-Behaved}
In this section, we prove Lemma~\ref{lem:goodSlikely}, i.e. that $\rS \sim \Dist^m_c$ rarely lies in $\bigcup_i \{ \sSet_i \cup \sSet'_i \cup \sSet''_i\}$. Lemma~\ref{lem:p1} already handles $\sSet_i$. For the remaining, we show the following two
\begin{lemma}
  \label{lem:p2}
If the constant $a>0$ is large enough, then for any $i$, we have $\Pr_{\rS \sim \Dist_c^m}[\rS \in \sSet'_i] \leq \delta/57$.
\end{lemma}

\begin{lemma}
  \label{lem:p3}
For any $i$, we have $\Pr_{\rS \sim \Dist_c^m}[\rS \in \sSet''_i] \leq \delta/57$.
\end{lemma}

Before we prove these last lemmas, observe that Lemma~\ref{lem:p1}, Lemma~\ref{lem:p2} and Lemma~\ref{lem:p3} together imply Lemma~\ref{lem:goodSlikely} by a union bound over all the $3 \cdot 19 = 57$ events $\rS \in \sSet_i, \rS \in \sSet_i'$ and $\rS \in \sSet_i''$. We now give the two remaining proofs. For ease of reading, we first restate the definitions of $\sSet'_i$ and $\sSet''_i$. Also, recall that we defined $E_{S,i}$ as the set of $(x,c(x))$ with $x \in \Xs$ for which $g_{S(\Sub_{w_i})}(x)c(x) \leq 5/6$.
\begin{itemize}
\item For $i=1,\dots,19$, let $\sSet_i'$ denote the set of all training sets $S$ for which $\Loss^{5/6}_\Dist(g_{S(\Sub_{w_i})}) \geq (a/627)(d + \ln(1/\delta))/20^\kappa$ and $|S(L_{i}) \cap E_{S.i}| < 10^{9}(d+\ln(57/\delta))$.
  \item For $i=1,\dots,19$, let $\sSet_i''$ denote the set of all training sets $S$ for which $|S(L_{i}) \cap E_{S,i}| \geq 10^{9}(d+\ln(57/\delta))$ and there is a hypothesis $h \in \Hyp$ such that $h(x)=c(x)$ for all $(x,c(x)) \in S(L_{i}) \cap E_{S,i}$ and $\Loss_{\Dist(\cdot \mid E_{S,i})}(h) > 1/62700$.
\end{itemize}

\begin{proof}[Proof of Lemma~\ref{lem:p2}]
  Observe that all $I \in \Sub_{w_i}$ are disjoint from the indices in $L_{i}$ by definition of \textit{BuildBucket}$_d$. Hence for all $I \in \Sub_{w_i}$, we have that $\rS(I)$ is completely determined from $\rS([m] \setminus L_i)$. Let us use the notation $\rS_1 = \rS([m] \setminus L_i)$ and $\rS_2 = \rS(L_i)$. We will also write $g_{\rS(\Sub_{w_i})} = g_{\rS_1(\Sub_{w_i})}$ to make clear that $g_{\rS(\Sub_{w_i})}$ is completely determined from $\rS_1$. We likewise write $E_{\rS,i}=E_{\rS_1,i}$ as this set also depends only on $\rS_1$.

Now let $S_1$ be any outcome of $\rS_1$ for which $\Loss^{5/6}_\Dist(g_{S_1(\Sub_{w_i})}) \geq (a/627)(d + \ln(1/\delta))/20^\kappa$.
The samples $\rS_2$ are independent of the outcome $\rS_1=S_1$. Thus each sample in $\rS_2$ belongs to $E_{S_1,i}$ with probability at least $(a/627)(d + \ln(1/\delta))/20^\kappa$ and this is independent across the samples. Thus $\E_{\rS_2 \sim \Dist_c^{|L_i|}}[|\rS_2 \cap E_{S_1,i}|] \geq |L_i|(a/627)(d + \ln(1/\delta))/20^\kappa = 20^{\kappa-1}(a/627)(d + \ln(1/\delta))/20^\kappa = (a/12540)(d + \ln(1/\delta))$. For $a \geq 10^{16}$, this is at least $2 \cdot 10^{9} (d + \ln(57/\delta))$. A Chernoff bound implies that
\[
  \Pr_{\rS_2 \sim \Dist_c^{|L_i|}}[|\rS_2 \cap E_{S_1,i}| <  10^{9} (d+\ln(57/\delta))] < \exp(-10^{9}(d+\ln(57/\delta))/8) < \delta/57.
\]
Since this holds for every outcome $S_1$ of $\rS_1$ where $\Loss^{5/6}_\Dist(g_{S_1(\Sub_{w_i})}) \geq (a/627)(d + \ln(1/\delta))/20^\kappa$, the conclusion follows.
\end{proof}
To prove the last lemma, we need the following classic result by Vapnik~\cite{vapnik:estimation} with slightly better constants due to Blumer et al.~\cite{blumer1989learnability}
\begin{theorem}[Blumer et al.~\cite{blumer1989learnability}]
\label{thm:vc}
  For any $0 < \delta < 1$ and any distribution $\Dist$ over $\Xs$, it holds with probability at least $1-\delta$ over a set $\rS \sim \Dist_c^m$ that every $h \in \Hyp$ with $h(x)=c(x)$ for all $(x,c(x)) \in \rS$ satisfies
  \[
    \Loss_\Dist(h) \leq (2/m)(d \lg_2(2em/d) + \lg_2(2/\delta)).
  \]
\end{theorem}

\begin{corollary}
  \label{cor:vc}
  For any $0 < \delta < 1$ and any distribution $\Dist$ over $\Xs$, it holds with probability at least $1-\delta$ over a set $\rS \sim \Dist_c^m$ with $m\geq 10^{9}(d + \ln(1/\delta))$ that every $h \in \Hyp$ with $h(x)=c(x)$ for all $(x,c(x)) \in \rS$ satisfies
  \[
    \Loss_\Dist(h) \leq 1/62700.
  \]
\end{corollary}
\begin{proof} The proof can be found in the appendix. It is merely a matter of inserting the value of $m$ in Theorem~\ref{thm:vc}.
\end{proof}

\begin{proof}[Proof of Lemma~\ref{lem:p3}]
We re-use the notation from the proof of Lemma~\ref{lem:p2} where we let $\rS_1 = \rS([m] \setminus L_i)$ and $\rS_2 = \rS(L_i)$, again exploiting that $E_{\rS,i} = E_{\rS_1,i}$ and $g_{\rS(\Sub_{w_i})}=g_{\rS_1(\Sub_{w_i})}$ are fully determined from $\rS_1$.

Now fix any possible outcome $S_1$ of $\rS_1$. This also fixes $E_{S_1,i}$ and $g_{S_1(\Sub_{w_i})}$. For $\rS_2 \sim \Dist_c^{|L_i|}$, let $\rInto \subseteq L_i$ be the subset of indices $j \in L_i$ for which $\rS(\{j\}) \in E_{S_1,i}$ (note $\rS(\{j\}) \in \rS_2$). Now consider any outcome $R$ of $\rInto$ for which $|R| \geq 10^{9}(d + \ln(57/\delta))$. Conditioned on $\rInto=R$, the samples $\rS(R) \subseteq \rS_2$ are i.i.d. from the conditional distribution $\Dist(\cdot \mid E_{S_1,i})$. It follows from Corollary~\ref{cor:vc} that, conditioned on $\rInto=R$, it holds with probability at least $1-\delta/57$ over $\rS(R)$ that every $h \in \Hyp$ with $h(x)=c(x)$ for all $(x,c(x)) \in \rS(R)$ have $\Loss_{\Dist(\cdot \mid E_{S_1,i})}(h) \leq 1/62700$. Since $\rS(R) \subseteq \rS_2$, the same holds for all $h \in \Hyp$ with $h(x)=c(x)$ for every $(x,c(x)) \in \rS_2 = \rS(L_i)$.

Since this holds for all $S_1$ and $R$ with $|R| \geq 10^{9} (d + \ln(57/\delta))$, we conclude that $\Pr_{\rS \sim \Dist^m_c}[\rS \in \sSet_i''] \leq \delta/57$.
\end{proof}

\section{Conclusion}
\label{sec:conclusion}
In this work, we have shown that the classic bagging heuristic is an optimal PAC learner in the realizable setting if we sample just a logarithmic number of bootstrap samples.

As mentioned in Section~\ref{sec:intro}, we can also combine our proof with the recent work by~\cite{Larsen2022OptimalWT} to obtain a simpler optimal weak to strong learner. Let us comment further on this here. Said briefly, a $\gamma$-weak learner, is a learning algorithm that from a large enough constant number of samples from an unknown data distribution $\Dist$, with constant probability outputs a hypothesis $h$ with $\Loss_\Dist(h) \leq 1/2-\gamma$, i.e. it has an advantage of $\gamma$ over random guessing. On the other hand, an $(\eps,\delta)$-strong learner, is a learning algorithm that given enough samples $m(\eps,\delta)$, with probability $1-\delta$ outputs a hypothesis $h$ with $\Loss_\Dist(h) \leq \eps$~\cite{kearns1988learning,kearns1994cryptographic}. In the recent work~\cite{Larsen2022OptimalWT}, the authors showed that if we use the \textit{Sub-Sample} procedure from~\cite{hanneke2016optimal} (see Algorithm~\ref{alg:hanneke}), but instead of empirical risk minimization, run a version of AdaBoost called AdaBoost$^*_\nu$ by~\cite{ratsch2005efficient} on the sub-samples, then this combination gives a weak to strong learner with an optimal sample complexity of $O(d/(\eps \gamma^2) + \lg(1/\delta)/\eps)$. Here $d$ is the VC-dimension of the hypothesis set from which the weak learner outputs. Examining their proof, we can directly substitute Hanneke's \textit{Sub-Sample} by bagging, as our new proof boils down to applying Hanneke's reasoning on the buckets $\Buc_i$. We thus get an optimal weak to strong learner from the combination of the two classic concepts of bagging  and boosting.

Let us also make another comparison to Hanneke's algorithm. In his work, the number of sub-samples is independent of $\delta$, whereas our result for bagging requires $\Omega(\ln(m/\delta))$ sub-samples. Indeed, as bagging is a randomized algorithm, there will be some non-zero contribution to the failure probability due to the number of sub-samples (with some non-zero probability, all sub-samples are the same). We find it an interesting open problem whether the analysis can be tightened to yield a better dependency on $\delta$ in the number of sub-samples needed. We also mention that $O(m)$ sub-samples suffice for an optimal dependency on $\delta$ for any $0< \delta < 1$. This is because $O(m)$ sub-samples suffice for any $\delta \geq \exp(-O(m))$, and for even smaller $\delta$, we have that  $\ln(1/\delta)/m > 1$.

Finally, let us comment on the unspecified constant $a > 0$ in Theorem~\ref{thm:mainintro}. In our proof, we did not attempt to minimize $a$ and indeed it is rather ridiculous. With some care, one could certainly shave several orders of magnitude, but at the end of it, we still rely on Hanneke's proof which also does not have small constants. We believe that bagging actually provides quite good constants, but this would require a different proof strategy. At least our reduction to understanding $\Loss^{1/3}_\Dist(g_S)$ is very tight and incurs almost no loss in the constant $a$. A good starting point for improvements would be to find an alternative proof that $\Loss^{1/3}_\Dist(g_\rS)$ is small with high probability over $\rS$.

\bibliographystyle{abbrv}
\bibliography{refs}

\appendix

\section{Deferred Proofs}
Here we given the omitted proofs of the claims only needing standard arguments.

\begin{proof}[Proof of Lemma~\ref{lem:sizesample}]
  We need to show that $1/6 \geq \Pdist(0) = \Pr_{\rInd \sim [m]^n}[|\Distinct(\rInd)| \notin  [0.01m, 0.9m]]$. Consider first $\Pr_{\rInd \sim [m]^n}[|\Distinct(\rInd)| < 0.01m]$. If $|\Distinct(\rInd)| < 0.01m$, then there must be a set $Q \subseteq [m]$ of cardinality $0.01m$ with $\Distinct(I) \subseteq Q$. For a fixed $Q$, this happens with probability precisely $(|Q|/m)^n \leq (0.01)^{0.02 m}$. A union bound over all $\binom{m}{0.01m} \leq (100 e)^{0.01m}$ choices for $Q$ implies $\Pr_{\rInd \sim [m]^n}[|\Distinct(\rInd)| < 0.01m] \leq (e/100)^{0.01m}$. For $m \geq 100$, this is at most $1/30$. Consider next $\Pr_{\rInd \sim [m]^n}[|\Distinct(\rInd)| \geq 0.9m]$. This probability is only non-zero for $n \geq 0.9m$ so we assume this from hereon. If $|\Distinct(\rInd)| \geq 0.9m$, then there must be a set $Q$ of cardinality $0.9m$ such that $Q \subseteq \Distinct(\rInd)$. For a fixed such $Q$, any $I$ with $Q \subseteq \Distinct(I)$ can be uniquely described by first specifying the first occurrence of each $i \in Q$ as an index into $I$ and then specifying the remaining indices in $I$ one at a time. There are thus no more than $\binom{n}{0.9m} (0.9m)! m^{n-0.9m}$ such $I$. Hence $\Pr_{\rInd \sim [m]^n}[Q \subseteq \Distinct(I)] \leq \binom{n}{0.9m} (0.9m)! m^{-0.9m} \leq \binom{m}{0.9m} (0.9m)! m^{-0.9m} = m!/((0.1m)! m^{0.9m})$. Using Stirling's approximation, this is no more than $\sqrt{10 e} \cdot m^me^{-m}/((0.1m)^{0.1m} e^{-0.1m} m^{0.9m}) \leq 6 \cdot 10^{0.1m} e^{-0.9m} \leq 6 \cdot  e^{0.3m} e^{-0.9m} = 6 \cdot e^{-0.6m}$. A union bound over all $\binom{m}{0.9m} = \binom{m}{0.1m} \leq (10 e)^{0.1 m} \leq e^{0.4 m}$ choices for $Q$ implies $\Pr_{\rInd \sim [m]^n}[|\Distinct(\rInd)| \geq 0.9m] \leq 6 \cdot e^{-0.2 m}$. For $m \geq 100$, this is much less than $1/30$. We finally conclude $\Pdist(0) \leq 1/30 + 1/30 < 1/6$ as claimed.
\end{proof}

\begin{proof}[Proof of Corollary~\ref{cor:vc}]
  We see that for $m \geq 10^{9}(d + \ln(1/\delta))$, we have both
  \begin{eqnarray*}
    \frac{2 (d \lg_2(2em/d) + 1)}{m} &\leq& \frac{2 d \lg_2(4e10^9 d/d)}{10^9 d} \\
                               &=& \frac{2 \lg_2(4e10^9)}{10^9} \\
    &\leq& 1/125400
  \end{eqnarray*}
  and
  \begin{eqnarray*}
    \frac{2 \lg_2(1/\delta)}{m} &\leq& \frac{2 \lg_2(1/\delta)}{10^9\ln(1/\delta)} \\
                                &\leq& \frac{2 \lg_2(e)}{10^9} \\
    &\leq& 1/125400.
  \end{eqnarray*}
Thus $(2/m)(d\lg_2(2em/d)+\lg_2(2/\delta)) = (2/m)(d\lg_2(em/d)+\lg_2(1/\delta) + 1)  \leq 2/125400 = 1/62700$.
\end{proof}

\end{document}